\documentclass{article}

\usepackage{microtype}
\usepackage{graphicx}
\usepackage{subcaption}
\usepackage{tikz}
\usepackage{stfloats}
\usepackage{mathtools}
\usetikzlibrary{shapes,arrows.meta, positioning}
\usepackage{booktabs} 

\usepackage{hyperref}

\usepackage[accepted]{icml2025}

\usepackage{amsmath}
\usepackage{amssymb}
\usepackage{amsthm}

\usepackage[capitalize,noabbrev]{cleveref}

\theoremstyle{plain}

\theoremstyle{definition}

\theoremstyle{remark}

\usepackage{thmtools} 
\usepackage{mathabx}
\usepackage{thm-restate}

\usepackage[textsize=tiny]{todonotes}

\icmltitlerunning{Exponential Family Variational Flow Matching for Tabular Data Generation}

\begin{document}

\twocolumn[
\icmltitle{Exponential Family Variational Flow Matching for Tabular Data Generation}



\icmlsetsymbol{equal}{*}

\begin{icmlauthorlist}
\icmlauthor{Andr\'es Guzm\'an-Cordero}{equal,vect,comp}
\icmlauthor{Floor Eijkelboom}{equal,comp}
\icmlauthor{Jan-Willem van de Meent}{comp}

\end{icmlauthorlist}

\icmlaffiliation{vect}{Vector Institute}
\icmlaffiliation{comp}{Bosch-Delta Lab}

\icmlcorrespondingauthor{Andrés Guzmán-Cordero}{andresguzco@gmail.com}
\icmlcorrespondingauthor{Floor Eijkelboom}{f.eijkelboom@uva.nl}

\icmlkeywords{Machine Learning, ICML}

\vskip 0.3in
]



\printAffiliationsAndNotice{\icmlEqualContribution}

\begin{abstract}
While denoising diffusion and flow matching have driven major advances in generative modeling, their application to tabular data remains limited, despite its ubiquity in real-world applications. 
To this end, we develop TabbyFlow, a variational Flow Matching (VFM) method for tabular data generation. 
To apply VFM to data with mixed continuous and discrete features, we introduce Exponential Family Variational Flow Matching (EF-VFM), which represents heterogeneous data types using a general exponential family distribution. 
We hereby obtain an efficient, data-driven objective based on moment matching, enabling principled learning of probability paths over mixed continuous and discrete variables. 
We also establish a connection between variational flow matching and generalized flow matching objectives based on Bregman divergences. 
Evaluation on tabular data benchmarks demonstrates state-of-the-art performance compared to baselines.
\end{abstract} 
\section{Introduction}
Generative modeling has become a fundamental task in machine learning, enabling the synthesis of complex and high-fidelity data across various domains. At the forefront of this evolution, different frameworks have been proposed, which focus on classical data modalities (e.g., images, text) for deep learning \citep{ramesh2022hierarchical, rombach2022high}. Flow-based approaches have shown remarkable progress in recent research, demonstrating increasing effectiveness and scalability. Building on foundations in continuous normalizing flows (CNF) \citep{chen2018neural, song2021maximum}, these methods have steadily advanced in their capabilities, though often requiring significant computational resources \cite{ben2022matching, rozen2021moser, grathwohl2018scalable}.

Among these approaches, Flow Matching has emerged as a particularly promising direction, proposing to learn a conditional vector field that can transport a source distribution to a target distribution in a simulation-free manner \cite{lipman2023flow}. To do this, it uses an interpolation of noise and real data. This framework has been further expanded to general geometries \citep{chen2024flow}, discrete probability paths \citep{gat2024discreteflowmatching}, and different applications \citep{wildberger2024flow, dao2023flow, kohler2023flow}. Recent theoretical work has established connections between flow matching and other generative approaches, showing equivalences under certain conditions that help unify the understanding of these methods \citep{albergo2023stochastic}. Variational Flow Matching (VFM) \citep{eijkelboom2024variational} generalizes flow matching as a general variational inference problem over trajectories induced by the used interpolation in flow matching. 
Recent theoretical work has also highlighted deep connections between flow matching, score-based methods, and likelihood training, showing that these approaches can be viewed as special cases within a broader variational framework \citep{albergo2023stochastic}.

In this paper, we consider the application of flow matching to the modeling of tabular data, a ubiquitous data modality in a range of domains including finance, healthcare, and marketing. Tabular data pose unique modeling challenges due to heterogeneous features, potential missing values, and varying scales \citep{survey, Wang2024}. While some diffusion-based approaches have been developed \citep{kotelnikov23a, tabsyn, TabDiff}, models for tabular data are less widespread than their counterparts for images and text. In this context, recent work on variational flow matching \citep{eijkelboom2024variational, zaghen2025towards, eijkelboom2025controlled} presents a promising avenue for generating mixed continuous and discrete features. VFM frames flow matching as an inference problem, in which the goal is to learn a variational posterior over data points that can be associated with the current interpolation point. This presents an opportunity to model heterogeneous data, which can be represented in a conceptually straightforward manner by matching the variational distribution to the data type for each variable. 

To realize this potential, we propose Exponential Family Variational Flow Matching (EF-VFM), an extension of VFM that incorporates exponential family distributions. Tabular data inherently consists of mixed variable types -- continuous, categorical, and binary -- necessitating a framework that can model these structures in a principled and unified way. Exponential families provide a natural solution by parameterizing data through sufficient statistics, enabling direct integration into the VFM paradigm. This approach not only facilitates efficient training via moment matching but also establishes deep connections between VFM and a generalized flow matching objective through the lens of Bregman divergences, offering a theoretical foundation for learning probability paths over mixed data types. To demonstrate the effectiveness of EF-VFM, we introduce TabbyFlow, a model that achieves state-of-the-art performance on standard tabular benchmarks, improving both fidelity and diversity in synthetic data generation.

\begin{figure*}
    \centering
    \includegraphics[width=\linewidth]{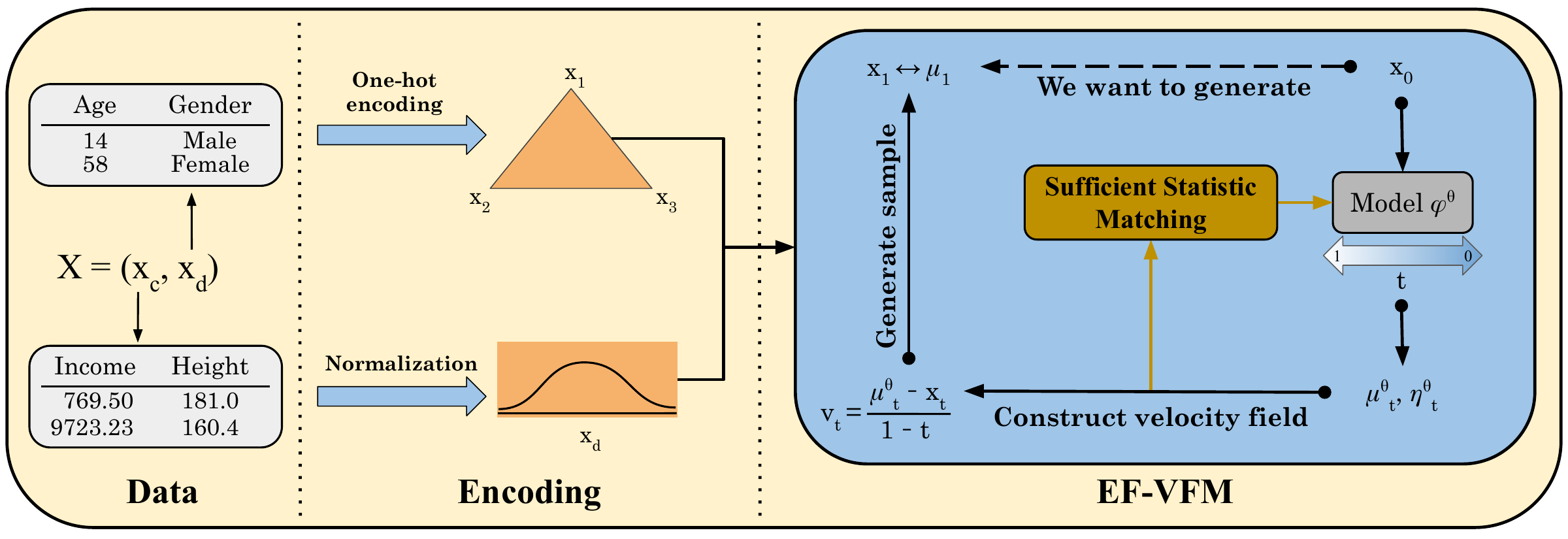}
    \caption{\textbf{Exponential Family Variational Flow Matching (EF-VFM)} is a generative modeling framework designed for mixed continuous and discrete variables. By leveraging the \textit{exponential family} and a mean-field assumption, EF-VFM efficiently matches the sufficient statistics of the distributions via learned probability paths, ensuring state-of-the-art fidelity and diversity in synthetic data.}
    \label{fig:abstract}
\end{figure*}
\section{Background}

\subsection{Transport Framework for Generative Modeling}

A central perspective in modern generative modeling interprets the task of sampling from a target distribution $p_1$ as transporting a base distribution $p_0$ along a (continuous) time path. Typically, $p_0$ is chosen to be a simple, tractable distribution, such as a standard Gaussian on $\mathbb{R}^D$. The transformation is defined through a time-dependent mapping
\begin{equation}
    \varphi_t \colon [0, 1] \times \mathbb{R}^D \to \mathbb{R}^D,
\end{equation}
where $\varphi_0$ is the identity map and $\varphi_1$ pushes $p_0$ onto $p_1$. In normalizing flows, this evolution is governed by an ordinary differential equation (ODE):
\begin{equation}
    \frac{\mathrm{d}}{\mathrm{d}t} \varphi_t(x) = u_t(\varphi_t(x)),
    \quad
    \varphi_0(x) = x,
\end{equation}
where $u_t: [0, 1] \times \mathbb{R}^D \to \mathbb{R}^D$ is a time-dependent velocity field. Given a parameterized model $v_t^\theta$ (e.g., a neural network), the goal is to approximate the transport dynamics that map samples from $p_0$ to $p_1$.

If $u_t$ is locally Lipschitz, the ODE admits a unique global solution, ensuring invertibility of $\varphi_t$. This allows for density estimation via the change-of-variables formula, which underpins likelihood-based training in normalizing flows. However, solving ODEs during training is computationally expensive, motivating alternative approaches that avoid explicit integration.

\subsection{Flow Matching}

Flow Matching (FM) circumvents the need for solving ODEs during training by directly learning the velocity field via regression:
\begin{equation}
    \mathcal{L}_{\mathrm{FM}}(\theta) =
    \mathbb{E}_{t,x}\bigl[
        \| u_t(x) - v_t^\theta(x) \|^2
    \bigr].
\end{equation}
Here, the expectation is typically uniform in $t \in [0,1]$ and sampled from the law of $\varphi_t(x_0)$ with $x_0 \sim p_0$. While $u_t$ is not known explicitly, it can be expressed in terms of a \textit{conditional velocity field} $u_t(x \mid x_1)$, which describes the motion of $x$ towards a designated endpoint $x_1$. The marginal velocity is then given by:
\begin{equation}
    u_t(x) =
    \mathbb{E}_{p_t(x_1 \mid x)}
    \bigl[ u_t(x \mid x_1) \bigr],
\end{equation}
where $p_t(x_1 \mid x)$ is the (unknown) posterior over endpoints. In practice, one can estimate $u_t(x)$ via Monte Carlo sampling of $u_t(x \mid x_1)$, leading to the \textit{Conditional Flow Matching} (CFM) objective:
\begin{equation}
    \mathcal{L}_{\mathrm{CFM}}(\theta) =
    \mathbb{E}_{t,x_1,x}
    \left[ \| u_t(x \mid x_1) - v_t^\theta(x) \|^2 \right].
\end{equation}
A key property of CFM is that minimizing this conditional loss provides an unbiased gradient estimate of $\mathcal{L}_{\mathrm{FM}}(\theta)$, ensuring that the learned model approximates the true marginal velocity. By avoiding ODE integration during training, Flow Matching offers a computationally efficient alternative to diffusion and score-based models.

\subsection{Flow Matching as a Variational Inference Problem}

Variational Flow Matching (VFM) extends Flow Matching by introducing a variational distribution $q_t^\theta(x_1 \mid x)$ in place of the unknown posterior $p_t(x_1 \mid x)$. Instead of estimating $u_t(x)$ via an expectation over $p_t(x_1 \mid x)$, VFM replaces it with an expectation over $q_t^\theta$:
\begin{equation}
    v_t^\theta(x) =
    \mathbb{E}_{q_t^\theta(x_1 \mid x)}
    \bigl[ u_t(x \mid x_1) \bigr].
\end{equation}
To ensure that $q_t^\theta(x_1 \mid x)$ accurately approximates the posterior $p_t(x_1 \mid x)$, VFM minimizes the KL divergence between their joint distributions:
\begin{align}
    \mathcal{L}_{\mathrm{VFM}}(\theta) &= 
     \mathbb{E}_t \bigl[ \text{KL}\bigl(p_t(x_1, x) ~\|~ q_t^{\theta}(x_1, x) \bigr) \bigr]
    \\ &= -\mathbb{E}_{t, x_1, x} \bigl[ \log q_t^\theta(x_1 \mid x) \bigr] + \text{const.}
\end{align}
Minimizing $\mathcal{L}_{\mathrm{VFM}}(\theta)$ ensures that $q_t^\theta(x_1 \mid x)$ approximates $p_t(x_1 \mid x)$, causing $v_t^\theta(x)$ to converge to $u_t(x)$.

This variational formulation is particularly effective when $u_t(x \mid x_1)$ is linear in $x_1$, e.g., if $u_t$ corresponds to straight-line interpolation of diffusion, as then we obtain:
\[
    \mathbb{E}_{p_t(x_1\mid x)}[ u_t(x\mid x_1) ]
    =
    u_t\bigl( x \mid \mathbb{E}_{p_t(x_1\mid x)}[ x_1 ] \bigr).
\]
This implies that modeling only the mean of $p_t(x_1 \mid x)$ under $q_t^\theta$ suffices to recover the true marginal velocity. Consequently, VFM can be implemented with a mean-field parameterization without loss of generality:
\begin{equation}
    \mathcal{L}_{\mathrm{MF-VFM}}(\theta) =
    -\mathbb{E}_{t, x_1, x} \left[
    \sum_{d=1}^{D} \log q_t^{\theta}(x_1^d \mid x) \right].
\end{equation}
This factorization simplifies learning, reducing the problem of estimating a high-dimensional distribution to learning $D$ univariate distributions.

A key advantage of VFM is its flexibility in choosing $q_t^\theta$. By selecting the `right' distributions -- such as categorical distributions for discrete features or Gaussian for continuous ones -- VFM provides a unified treatment of discrete and continuous generative modeling. This generalization, combined with its computational efficiency, makes VFM particularly well-suited for tasks of mixed modality.

\subsection{Exponential Family}

The exponential family is a class of probability distributions that is widely used in statistics and machine learning due to its mathematical convenience and flexibility. Its structure simplifies parameter estimation, enables efficient inference, and unifies many commonly used distributions under a single framework. A distribution belongs to the exponential family if it can be written in the form
\begin{equation}
p(x \mid \eta) = h(x) \exp\left(\tau(x) \cdot \eta - A(\eta)\right),
\end{equation}
where $h(x)$ is the base measure, $\eta$ are the natural parameters, $\tau(x)$ represents the sufficient statistics, and $A(\eta)$ is the log-partition function ensuring normalization. Examples of distributions in this family include the Gaussian, Bernoulli, Poisson, and exponential distributions.

Exponential families can be understood through two equivalent parameterizations: the natural parameters $\eta$ and the mean parameters $\mu$. While natural parameters define the exponential family form directly, mean parameters arise from expectations of the sufficient statistics, i.e.
\begin{equation}
    \mu = \mathbb{E}[\tau(x)] = \nabla A(\eta),
\end{equation}
where $\nabla A(\eta)$ is the gradient of the log-partition function. This relationship can be inverted through the conjugate dual $A^*$ of the log-partition function, giving $\eta = \nabla A^*(\mu)$.
The mean parameterization is particularly useful in practice as mean parameters often have direct interpretations -- for instance, in a Gaussian distribution, they correspond to the mean and variance. This makes them especially valuable for maximum likelihood estimation, where empirical sufficient statistics directly estimate the mean parameters.

Throughout this work, we assume our exponential families are minimal, meaning their sufficient statistics are linearly independent. This ensures the relationship between natural and mean parameters is bijective, enabling the clean theoretical results and practical algorithms that follow.

\section{Exponential Family VFM}

\subsection{Motivation}
Tabular data presents a unique challenge for generative modeling: each column may contain different types of data -- continuous, categorical, or binary -- all of which must be modeled jointly. The key insight of our work is that exponential families provide a natural framework for extending VFM to tabular data. Exponential families offer two critical advantages in this setting. First, they include distributions suitable for each data type commonly found in tables -- Gaussian distributions for continuous variables like age or income, categorical distributions for discrete variables like education level or occupation, Bernoulli distributions for binary indicators like purchase history or customer status, and Poisson or exponential distributions for count or time-based data. Second, they admit a unified mathematical treatment through their mean parameterization, which, as we show in \cref{ssec:exp_fam_vfm}, enables efficient training through sufficient statistics matching.
This exponential family perspective not only provides practical benefits for tabular data generation but also deepens our theoretical understanding of flow matching. As we demonstrate in \cref{ssec:bregman_connection}, it reveals a fundamental connection between VFM and Bregman divergences, providing a principled justification for our approach while establishing links to classical flow matching objectives.

\subsection{Exponential Families for VFM}
\label{ssec:exp_fam_vfm}

\paragraph{Moment Matching.}
To handle the diverse data types present in tabular data, we use exponential families as our variational distributions, as they naturally accommodate both continuous and discrete variables. For each column type, we learn a parameterization (either through natural or mean parameters) using a neural network $\theta$. Formally, this distribution can be written as
\begin{equation}
    q_t^\theta(x_1 \mid x_t) = \exp\left(\tau(x_1) \cdot  \eta_t^{\theta}(x) - A(\eta^{\theta}_t(x))\right).
\end{equation}
As such, the VFM loss (as given by a log-likelihood style objective) reduces to
\begin{align}
    \mathcal{L}(\theta) &= -\mathbb{E}_{t, x_1, x} \left[\log q_t^{\theta}(x_1 \mid x)\right] \\
    &= -\mathbb{E}_{t, x_1, x} \left[ \tau(x_1) \cdot \eta_t^{\theta}(x)  - A(\eta_t^{\theta}(x)) \right].
\end{align}
A key advantage of this formulation for tabular data is that optimizing this loss reduces to a simple and efficient objective for training VFM, which we term \textit{statistics matching}. The gradient of the loss depends solely on the difference between the empirical sufficient statistics and the model's predicted sufficient statistics - a property that proves crucial for handling heterogeneous column types efficiently. This reduction in complexity not only streamlines the training process but also enhances scalability across columns with different distributions, allowing for the practical application of VFM to real-world tabular datasets. Formally, the following holds:
\begin{restatable}{restatable_proposition}{gradientefvfm}
\label{thm:gradientefvfm}
Let $q_t^{\theta}(x_1 \mid x)$ be a variational distribution from an exponential family, parameterized by natural parameters $\eta^{\theta}_t(x)$, which depend on neural network parameters $\theta$. The gradient of the VFM objective  $\nabla_{\theta} \mathcal{L}(\theta)$ 
is:
\begin{equation}
 - \mathbb{E}_{t, x_1, x} \left[
    \left(\mu_t(x) - \mu_t^\theta(x) \right)  
    \cdot \nabla_{\theta} \eta_t^{\theta}(x)
\right],
\end{equation}
where $\mu_t(x) = \mathbb{E}_{p_t(x_1 \mid x)}[\tau(x_1)]$ are the moments relative to $p_t(x_1 \mid x)$, and $\mu_t^\theta(x) = \mathbb{E}_{q^\theta_t(x_1 \mid x)}[\tau(x_1)]$ are the moments relative to the variational approximation.
\end{restatable}
\begin{proof}
    See \cref{appendix:gradientefvfm}. It follows from this identity that the gradient is zero when the moment-matching condition $\mu_t^\theta(x) = \mu_t(x)$ is satisfied.
\end{proof}

\paragraph{Simplification under Linear Conditional Velocity Fields.}
If the conditional velocity field $u_t(x \mid x_1)$ is linear in $x_1$, it suffices to match only the mean of the posterior distribution $p_t(x_1 \mid x_t)$. In this case, the marginal velocity field simplifies to  
\begin{equation}
    u_t(x) = \mathbb{E}\left[u_t(x \mid x_1)\right] = u_t(x \mid \mathbb{E}[x_1]).
\end{equation}
This reduces the complexity of the optimization problem: rather than working with the full posterior distribution, one only needs to learn the parameters associated with the sufficient statistics $\tau(x_1) = x_1$. 

We assume that $q(x_1 \mid x)$ belongs to an exponential family, where $\tau(x_1) = x_1$ serves as a sufficient statistic. This assumption holds in many practical cases, as $\mathbb{E}[x_1]$ can often be computed in closed form, allowing for efficient parameter updates. For example, in a Gaussian setting, $x_1$ directly corresponds to the mean, aligning naturally with the moments of the distribution. Care is needed when the sufficient statistics differ. In a categorical distribution, for instance, the sufficient statistics are given by $\tau_k(x) = \mathbb{I}[x = k]$, where $\mathbb{I}[\cdot]$ is the indicator function. Here, the assumption $\tau(x_1) = x_1$ holds only if $x_1$ is represented as a one-hot vector, ensuring that $\tau_k(x) = x_k$. Without this representation, the correspondence between the learned parameters and the sufficient statistics would need to be reformulated to remain consistent with the exponential family framework.

The linearity assumption implies that the generative process is fully governed by the expected value of $x_1$. In the context of tabular data, this leads to intuitive loss functions for different column types:

\begin{itemize}
   \item \textbf{Categorical columns:} For columns like `occupation' or `education level', the objective simplifies to minimizing cross-entropy loss between predicted and empirical category probabilities.
   \item \textbf{Continuous columns:} For numerical columns like `age' or `income', the optimization reduces to minimizing the mean squared error between predicted and empirical means.
   \item \textbf{Binary indicators:} For yes/no columns like `has purchased' or `is subscriber', the objective naturally handles these as special cases of categorical variables with two states.
\end{itemize}

This unified treatment of different column types is crucial for tabular data generation. By structuring the problem through sufficient statistics matching, we can handle heterogeneous data types while maintaining computational efficiency - a key requirement given the typically modest size of tabular datasets. The approach naturally accommodates the mixed continuous and discrete variables found in real-world tables, while the connection to familiar loss functions makes the training process interpretable and robust.

\subsection{Connection to Flow Matching}
\label{ssec:bregman_connection}
\paragraph{VFM as Bregman Divergence Minimization.}
At first glance, the variational and standard flow matching objectives seem rather different. Rather than optimizing a log-likelihood, the flow matching objective involves minimizing a dissimilarity metric (typically MSE) between the conditional velocity field and the model. For tabular data, different column types naturally suggest different loss functions -- MSE for continuous variables, cross-entropy for categorical ones -- raising the question of how these diverse objectives relate. The VFM paper demonstrates that optimizing a specific Gaussian posterior distribution is equivalent to minimizing the MSE between the conditional velocity field and the model. Interestingly, this connection generalizes: each exponential family induces its own natural Bregman divergence, providing a unified framework for handling mixed data types. As we will show, this reveals a deep connection between the probabilistic view on flow matching and the conditional flow matching objective.

Many loss functions are derived from log-likelihood estimations of various distributions, e.g., the MSE from a Gaussian log-likelihood and cross-entropy from a categorical log-likelihood. For the exponential family, it is possible to derive a general loss function in terms of a \textit{Bregman} divergence, i.e., a divergence induced by a convex function $\psi$, defined as:
\begin{equation}
    D_{\psi}(u, v) := \psi(u) - \psi(v) - \left\langle u - v, \, \nabla_v \psi(v) \right\rangle.
\end{equation}
Note that indeed $D_{\psi}(u, v) \geq 0$ for all $u, v$, and moreover that $D_{\psi}(u, v) = 0$ if and only if $u = v$. Formally, we have the following result:

\begin{restatable}{restatable_proposition}{expfambreg}
\label{thm:expfambreg}
Let $q^{\theta}_t(x_1 \mid x)$ be an exponential family distribution that is regular and minimal. Then, the variational flow matching objective is equivalent to minimizing the Bregman divergence induced by the conjugate dual of the log normalizer evaluated between the sufficient statistics and predicted mean parameters. 
\end{restatable}
\begin{proof}
    See \cref{appendix:efvfmasbregman}.
\end{proof}

\paragraph{Connection to Flow Matching.}
One key property of Bregman divergences is that when taking the gradient in the second argument, the total divergence is invariant under affine (and thus, convex) combinations in the first argument, i.e., when $\sum_{i=1}^n \alpha_i = 1$, then
\begin{equation}
    \nabla_y D_{\psi}\left(\sum_{i=1}^n \alpha_i x_i, y\right) = \sum_{i=1}^n \alpha_i \nabla_y D_{\psi}(x_i, y).
\end{equation}
In particular, this implies that expectation can be taken either inside or outside the divergence, and hence for any random variable $x$, we have
\begin{equation}
    \nabla_y D_{\psi} \left( \mathbb{E} \left[x\right], y \right) = \mathbb{E} \left[ \nabla_y D_{\psi}(x, y) \right].
\end{equation}
As it turns out, it is \textit{exactly} this condition that allows one to switch from the marginal to conditional objectives in flow matching. That is, one can show that indeed the objectives
\begin{equation}
\nabla_{\theta} \mathcal{L}_{\text{FM}}(\theta) = \nabla_{\theta} \mathbb{E}_{t, x} \left[D_{\psi}(u_t(x), v_t^{\theta}(x)) \right]
\end{equation}
and 
\begin{equation}
\nabla_{\theta} \mathcal{L}_{\text{CFM}}(\theta) = \nabla_{\theta} \mathbb{E}_{t, x_1, x} \left[D_{\psi}(u_t(x \mid x_1), v_t^{\theta}(x)) \right]
\end{equation}
coincide exactly when  $D_{\psi}$ is a Bregman divergence, as recently shown by \citet{holderrieth2024generator}.

Recognizing that the Variational Flow Matching objective corresponds to minimizing a Bregman divergence provides a useful insight: it reveals that the fundamental `trick' in Flow Matching -- the ability to optimize using conditional trajectories instead of marginal ones -- emerges naturally from probabilistic inference. For tabular data generation, this is particularly valuable as it provides a principled way to handle different column types through their naturally induced divergences: squared error for continuous variables and (categorical) cross-entropy for categorical ones, and so on. This theoretical understanding has two important practical implications. First, it suggests that Flow Matching's workings are fundamentally connected to probabilistic inference and optimization. Second, it suggests natural extensions of Flow Matching to new settings by choosing appropriate Bregman divergences. This connection between Flow Matching and probabilistic inference opens new avenues for both theoretical analysis and practical improvements in generative modeling, particularly for heterogeneous data types.

\subsection{A Comment on Information Geometry and Natural Gradient Descent in Exponential Family VFM}
Information geometry studies probability distributions through the framework of differential geometry, interpreting the parameter spaces of these distributions as Riemannian manifolds. As learning takes place in this parameter space, which is not Euclidean \citep{Amari2016}, we need to take its structure into account when designing the learning method. The metric defining the geometry of such a manifold is the Fisher information matrix, given by:
\begin{equation}
    g_{ij}(\theta) = \mathbb{E}\left[ \frac{\partial \log p(x; \theta)}{\partial \theta_i} \frac{\partial \log p(x; \theta)}{\partial \theta_j} \right].
\end{equation}
This perspective allows for the utilization of the manifold's curvature in optimization processes.

In exponential family distributions, the Fisher information matrix $\mathrm{I}_F(\theta)$ is the Hessian of the log-partition function $A(\eta)$, i.e.
\begin{equation}
    \mathrm{I}_F(\theta) = \nabla^2 A(\eta).
\end{equation}
 Though not studied in this work explicitly, this relationship opens to possibility to apply e.g., natural gradient descent in flow models, an optimization method that adjusts parameter updates according to the manifold's curvature, potentially improving convergence. This seems especially promising regarding VFM on manifolds \cite{zaghen2025towards}.

\section{Experiments (TabbyFlow)}

To demonstrate our theoretical framework in practice, we introduce TabbyFlow, an implementation of EF-VFM specifically designed for tabular data generation. The key insight of TabbyFlow is its direct use of the exponential family perspective: each column in the table is modeled using the appropriate distribution from the exponential family. This approach naturally handles mixed data types while maintaining the efficiency benefits of sufficient statistics matching.
Our implementation uses a transformer architecture \cite{vaswani2017attention} to learn the parameters for each column's distribution through their sufficient statistics. Rather than requiring separate architectures or processing stages for different column types, TabbyFlow's unified treatment through exponential families allows for a simple, single-pass architecture (see \cref{app:implementation} for more information). Code can be found in \url{https://github.com/andresguzco/ef-vfm}.

\begin{table*}[t]
\caption{Performance comparison on the error rates (\%) of \textbf{Shape}.}
\centering
\resizebox{\textwidth}{!}{
\begin{tabular}{lccccccc}
\toprule
\textbf{Method} & \textbf{Adult} & \textbf{Default} & \textbf{Shoppers} & \textbf{Magic} & \textbf{Beijing} & \textbf{News} & \textbf{Average} \\
\midrule
CTGAN & $16.84 \pm 0.03$ & $16.83 \pm 0.04$ & $21.15 \pm 0.10$ & $9.81 \pm 0.08$ & $21.39 \pm 0.05$ & $16.09 \pm 0.02$ & $15.99$ \\
TVAE & $14.22 \pm 0.08$ & $10.17 \pm 0.05$ & $24.51 \pm 0.06$ & $8.25 \pm 0.06$ & $19.16 \pm 0.06$ & $16.62 \pm 0.03$ & $15.97$ \\
GOGGLE & $16.97$ & $17.02$ & $22.33$ & $1.90$ & $16.93$ & $25.32$ & $17.91$ \\
TabDDPM & $1.75 \pm 0.03$ & $1.57 \pm 0.08$ & $2.72 \pm 0.13$ & $1.01 \pm 0.09$ & $1.30 \pm 0.03$ & $78.75 \pm 0.01$ & $16.93$ \\
TabSyn & $0.81 \pm 0.05$ & $1.01 \pm 0.08$ & $1.44 \pm 0.07$ & $1.03 \pm 0.14$ & $1.26 \pm 0.05$ & $2.06 \pm 0.04$ & $1.35$ \\
TabDiff & $0.63 \pm 0.05$ & $1.24 \pm 0.07$ & $1.28 \pm 0.09$ & $\mathbf{0.78 \pm 0.08}$ & $\mathbf{1.03 \pm 0.05}$ & $2.35 \pm 0.03$ & $1.17$ \\
\midrule
\textbf{TabbyFlow} & $\mathbf{0.62 \pm 0.07}$ & $\mathbf{0.89 \pm 0.03}$ & $\mathbf{1.03 \pm 0.08}$ & $1.58 \pm 0.07$ & $1.05 \pm 0.05$ & $\mathbf{1.31 \pm 0.04}$ & $\mathbf{1.08}$ \\
\bottomrule
\end{tabular}}
\label{tab:shape_results}
\end{table*}

\begin{table*}[bp]
\centering
\caption{Performance comparison on the error rates (\%) of \textbf{Trend}.}
\resizebox{\textwidth}{!}{
\begin{tabular}{lccccccc}
\toprule
\textbf{Method} & \textbf{Adult} & \textbf{Default} & \textbf{Shoppers} & \textbf{Magic} & \textbf{Beijing} & \textbf{News} & \textbf{Average} \\
\midrule
CTGAN & $20.23 \pm 1.20$ & $26.95 \pm 0.93$ & $13.08 \pm 0.16$ & $7.00 \pm 0.19$ & $22.95 \pm 0.08$ & $5.37 \pm 0.05$ & $16.36$ \\
TVAE & $14.15 \pm 0.88$ & $19.50 \pm 0.95$ & $18.67 \pm 0.38$ & $5.82 \pm 0.49$ & $18.01 \pm 0.08$ & $6.17 \pm 0.09$ & $16.44$ \\
GOGGLE & $45.29$ & $21.94$ & $23.90$ & $9.47$ & $45.94$ & $23.19$ & $28.18$ \\
TabDDPM & $3.01 \pm 0.25$ & $4.89 \pm 0.10$ & $6.61 \pm 0.16$ & $1.70 \pm 0.22$ & $2.71 \pm 0.09$ & $13.16 \pm 0.11$ & $11.95$ \\
TabSyn & $1.93 \pm 0.07$ & $2.81 \pm 0.48$ & $2.13 \pm 0.10$ & $0.88 \pm 0.18$ & $3.13 \pm 0.34$ & $1.52 \pm 0.03$ & $2.33$ \\
TabDiff & $1.49 \pm 0.16$ & $\mathbf{2.55 \pm 0.75}$ & $1.74 \pm 0.08$ & $\mathbf{0.76 \pm 0.12}$ & $\mathbf{2.59 \pm 0.15}$ & $\mathbf{1.28 \pm 0.04}$ & $1.80$ \\
\midrule
\textbf{TabbyFlow} & $\mathbf{1.09 \pm 0.21}$ & $2.56 \pm 0.33$ & $\mathbf{1.58 \pm 0.15}$ & $1.07 \pm 0.26$ & $2.93 \pm 0.29$ & $1.38 \pm 0.09$ & $\mathbf{1.77}$ \\
\bottomrule
\end{tabular}}
\label{tab:trend_results}
\end{table*}

\subsection{Experimental setup}

\paragraph{Formalism.} We consider tabular datasets with $C_n$ numerical and $C_c$ categorical features. Each observation, denoted as  $\mathbf{x}^{(i)} = [\mathbf{x}_n^{(i)}, \mathbf{x}_c^{(i)}]$, consists of numerical features $\mathbf{x}_n^{(i)}\in \mathbb{R}^{C_n}$ and categorical features $\mathbf{x}_c^{(i)}$, where each categorical feature $x_{c_j}^{(i)}$ is represented as a one-hot vector on the probability simplex $\Delta^{C_{n_j}}$ in the case feature $j$ has $C_{n_j}$ categories.
By using an optimal transport interpolation which is linear in $x_1$, the VFM objective factorizes over dimensions:
\begin{equation*}
\mathcal{L}(\theta)
= -\mathbb{E}_{t, x_1, x} \left[\sum^D_{d=1} \log q^\theta_t(x^d_1 \mid x)\right],
\end{equation*}
allowing us to compute the vector field using the first moment of each one-dimensional distribution $q^\theta_t(x^d_1\mid x)$ independently.

\paragraph{Datasets.} 
We make use of seven well-known tabular datasets: Adult, Default, Shoppers, Magic, Fault, Beijing, News, and Diabetes. These datasets have various sizes, numbers of features, and distributions, and have been used before to evaluate generative models for tabular data. Each dataset contains a mixture of continuous and discrete variables (categorical or Bernoulli) and has a designated classification or regression downstream task. See \cref{app:data} for more details.

\begin{table*}[t]
\centering
\caption{Comparison of $\alpha$-Precision scores.}
\resizebox{\textwidth}{!}{
\begin{tabular}{lcccccccc}
\toprule
\textbf{Methods} & \textbf{Adult} & \textbf{Default} & \textbf{Shoppers} & \textbf{Magic} & \textbf{Beijing} & \textbf{News} & \textbf{Average} & \textbf{Ranking} \\
\midrule
CTGAN   & $77.74\pm0.15$  & $62.08\pm0.08$  & $76.97\pm0.39$  & $86.90\pm0.22$ & $96.27\pm0.14$ & $96.96\pm0.17$ & $82.82$ & $5$ \\ 
TVAE    & $98.17\pm0.17$  & $85.57\pm0.34$  & $58.19\pm0.26$  & $86.19\pm0.48$ & $97.20\pm0.10$ & $86.41\pm0.17$ & $85.29$ & $4$ \\ 
GOGGLE  & $50.68$         & $68.89$         & $86.95$         & $90.88$        & $88.81$        & $86.41$        & $78.77$ & $7$ \\ 
TabDDPM & $96.36\pm0.20$  & $97.59\pm0.36$  & $88.55\pm0.68$  & $88.59\pm0.17$ & $97.93\pm0.30$ & $-$  & $79.83$ & $6$ \\ 
TabSyn  & $\mathbf{99.52\pm0.10}$ & $99.26\pm0.27$ & $\mathbf{99.16\pm0.22}$ & $99.38\pm0.27$ & $98.47\pm0.10$ & $96.80\pm0.25$ & $98.67$ & $2$ \\
TabDiff & $99.02\pm0.20$ & $98.49\pm0.28$ & $99.11\pm0.34$ & $\mathbf{99.47\pm0.21}$ & $98.06\pm0.24$ & $\mathbf{97.36\pm0.17}$ & $98.59$ & $3$\\ 
\midrule
\textbf{TabbyFlow} 
        & $99.43\pm0.18$ 
        & $\mathbf{99.31\pm0.24}$ 
        & $98.96\pm0.10$ 
        & $98.27\pm0.39$ 
        & $\mathbf{98.93\pm0.15}$ 
        & $97.22\pm0.12$
        & $\mathbf{98.69}$ 
        & $\mathbf{1}$ \\
\bottomrule
\end{tabular}}
\label{tab:alpha}
\end{table*}

\begin{table*}[bp]
\caption{Comparison of $\beta$-Recall scores.}
\centering
\resizebox{\textwidth}{!}{
\begin{tabular}{lcccccccc}
\toprule
\textbf{Methods} & \textbf{Adult} & \textbf{Default} & \textbf{Shoppers} & \textbf{Magic} & \textbf{Beijing} & \textbf{News} & \textbf{Average} & \textbf{Ranking} \\
\midrule
CTGAN   & $30.80\pm0.20$  & $18.22\pm0.17$  & $31.80\pm0.35$  & $11.75\pm0.20$ & $34.80\pm0.10$ & $24.97\pm0.29$ & $25.39$ & $6$ \\
TVAE    & $38.87\pm0.31$  & $23.13\pm0.11$  & $19.78\pm0.10$  & $32.44\pm0.35$ & $28.45\pm0.08$ & $29.66\pm0.21$ & $28.72$ & $5$ \\
GOGGLE  & $8.80$          & $14.38$         & $9.79$          & $9.88$         & $19.87$        & $2.03$         & $10.79$ & $7$ \\
TabDDPM & $47.05\pm0.25$  & $47.83\pm0.35$  & $47.79\pm0.25$  & $48.46\pm0.42$ & $\mathbf{56.92\pm0.13}$ & $-$  & $41.34$ & $4$ \\
TabSyn  & $47.56\pm0.22$  & $48.00\pm0.35$  & $48.95\pm0.28$  & $\mathbf{48.03\pm0.23}$ & $55.84\pm0.19$ & $45.04\pm0.34$ & $48.90$ & $3$ \\
TabDiff & $\mathbf{51.64 \pm 0.20}$ & $\mathbf{51.09 \pm 0.25}$ & $\mathbf{49.75 \pm 0.64}$ & $48.01 \pm 0.31$ & $59.63 \pm 0.23$ & $42.10 \pm 0.32$ & $\mathbf{50.37}$ & $\mathbf{1}$\\
\midrule
\textbf{TabbyFlow} 
        & $48.13\pm0.27$ 
        & $48.34\pm0.21$ 
        & $49.15\pm0.17$ 
        & $46.71\pm0.29$ 
        & $56.12\pm0.09$ 
        & $\mathbf{47.62\pm0.21}$ 
        & $49.35$ 
        & $2$ \\
\bottomrule
\end{tabular}}
\label{tab:beta}
\end{table*}

\paragraph{Baselines.} 
We compare TabbyFlow against various types of models: GAN-based CTAGAN \citep{Xu2019ModelingTD}, VAE-based TVAE \citep{Xu2019ModelingTD} and GOGGLE \citep{liu2023goggle}, and Diffusion-based TabDDPM \citep{kotelnikov23a}, TabSyn \citep{tabsyn} and TabDiff \citep{TabDiff}. The number of models for tabular data generation is vast, thus, we have restricted the evaluation to the most common paradigms in generative modeling.

\begin{table}[h]
    \centering
    \caption{Wasserstein distance for target and learned distributions.}
    \resizebox{0.30\textwidth}{!}{
    \begin{tabular}{c c c}
        \toprule
        \textbf{Model} & \textbf{WD (Train)} & \textbf{WD (Test)} \\
        \midrule
        TVAE & $4.6 \pm 0.3$ & $4.9\pm 0.1$\\
        CTGAN & $7.8\pm 0.2$ & $7.7\pm 0.1$\\
        TabDDPM & $3.1\pm 0.6$ & $3.9\pm 0.5$\\
        TabSyn & $2.2 \pm 0.4$ & $3.0\pm 0.3$\\
        TabDiff & $2.4\pm 0.3$ & $2.9 \pm 0.2$\\
        \midrule
        \textbf{TabbyFlow} & $\mathbf{1.7 \pm 0.7}$ & $\mathbf{2.1\pm 0.4}$\\
        \bottomrule
    \end{tabular}}
    \label{tab:WD}
\end{table}

\paragraph{Metrics.} 
We assess the quality of the synthetic data from four perspectives using a set of metrics widely adopted in previous studies \citep{tabsyn, TabDiff}. 
First, we evaluate the Wasserstein distance (WD) between the target and the learned distributions, doing so separately for the train and test datasets (see \Cref{tab:WD}).
Second, low-order statistics are evaluated through column-wise density estimation errors (referred to as \textit{Shape}, see \Cref{tab:shape_results}) and pairwise column correlations errors (referred to as \textit{Trend}, see \Cref{tab:trend_results}), measuring the density of individual columns and the relationships between column pairs. 
Third, high-order metrics such as $\alpha$-precision and $\beta$-recall scores are used to evaluate the general fidelity and diversity of synthetic data (see \Cref{tab:alpha} and \Cref{tab:beta}). 
We further compute C2ST, which describes how difficult it is to tell apart the real data from the synthetic data. 
Finally, performance on downstream tasks is measured using \textit{machine learning efficiency} (MLE, see \Cref{tab:MLE}), which compares test accuracy on real data when models are trained on synthetic datasets, and \textit{distance to closest relative} (DCR, see \Cref{tab:dcr}), the minimum distance between a synthetic data point and every original point.

\subsection{Data Fidelity and Downstream Performance}
\paragraph{Shape and Trend.}
Our evaluation framework uses Shape and Trend metrics to assess synthetic data quality. Shape quantifies how well the synthetic data preserves each column's marginal density, using the Kolmogorov-Smirnov Test (KST) for numerical columns and Total Variation Distance (TVD) for categorical columns. Trend measures the preservation of inter-column relationships, employing Pearson correlation for numerical pairs and contingency similarity for categorical pairs.
The results, detailed in Tables \ref{tab:shape_results} and \ref{tab:trend_results}, demonstrate TabbyFlow's strong performance across both metrics. TabbyFlow achieves the highest Shape scores on five out of six datasets, outperforming the diffusion-based TABSYN by an average margin of 0.32\%, thereby indicating better preservation of individual column distributions. Similarly, TabbyFlow maintains its strong performance on the Trend metric, exceeding TabSyn's average score by 0.84\%, suggesting better preservation of relationships between columns.

\paragraph{Precision and Recall.}
To further evaluate the fidelity of the generated data, we compute $\alpha$-precision and $\beta$-recall. The former measures how closely the synthetic data resembles the true data distribution, while the latter evaluates its diversity. Together, they ensure the generated data is realistic and representative of the real dataset, capturing the full range of the true distributions. Results for these metrics can be found in Tables \ref{tab:alpha} and \ref{tab:beta}

Our evaluation shows that TabbyFlow achieves strong performance across datasets, leading in average $\alpha$-precision scores, which measure similarity to real data. While TabSyn performs marginally better (within 1\%) on two datasets, TabbyFlow maintains competitive performance across all cases. A similar pattern emerges for $\beta$-Recall, where TabbyFlow leads on average but is occasionally outperformed by baselines on specific datasets. In evaluating these metrics, $\alpha$-Precision first establishes data authenticity, while $\beta$-Recall measures coverage of the original dataset's modes. While several baselines demonstrate strong performance on both metrics, TabbyFlow stands out for achieving competitive results with a significantly simpler architecture. This combination of architectural simplicity and strong empirical performance positions TabbyFlow as a compelling approach for tabular data generation.

\paragraph{Detection: classifier two-sample test (C2ST).} 
We evaluate the distinguishability of synthetic from real data using a two-sample test based on logistic regression (C2ST). This detection score provides stronger discriminative power compared to our previous metrics. Results in Table \ref{tab:comparison3} reveal that while baseline performances vary considerably, TabbyFlow outperforms all methods. These findings align with our earlier metrics, where TabbyFlow matches or exceeds the performance of the strongest baselines across datasets.

\begin{table*}[t]
\centering
\caption{Detection score (C2ST) using logistic regression classifier.}
\resizebox{0.65\textwidth}{!}{
\begin{tabular}{lcccccc}
\toprule
\textbf{Method} & \textbf{Adult} & \textbf{Default} & \textbf{Shoppers} & \textbf{Magic} & \textbf{Beijing} & \textbf{News} \\
\midrule
CTGAN   & 0.5949    & 0.4875    & 0.7488    & 0.6728    & 0.7531    & 0.6947 \\
TVAE    & 0.6315    & 0.6547    & 0.2962    & 0.7706    & 0.8659    & 0.4076 \\
GOGGLE  & 0.1114    & 0.5163    & 0.1418    & 0.9526    & 0.4779    & 0.0745 \\
TabDDPM & 0.9755    & 0.9712    & 0.8349    & \textbf{0.9998}    & 0.9513    & 0.0002 \\
TabSyn & \textbf{0.9986} & 0.9870 & 0.9740 & 0.9732 & 0.9603 & 0.9749 \\
TabDiff & 0.9950 & 0.9774 & \textbf{0.9843} & 0.9989 & \textbf{0.9781} & 0.9308 \\
\midrule
\textbf{TabbyFlow} & 0.9953 & \textbf{0.9910} & 0.9810 & 0.9775 & 0.9466 & \textbf{0.9808} \\
\bottomrule
\end{tabular}
}
\label{tab:comparison3}
\end{table*}

\begin{table*}[bp]
\centering
\caption{Comparison of DCR across datasets.}
\resizebox{0.8\textwidth}{!}{
\begin{tabular}{lccccc}
\toprule
\textbf{Methods}   & \textbf{Adult}    & \textbf{Default}   & \textbf{Shoppers}  & \textbf{Beijing}   & \textbf{News}      \\
\midrule
TabDDPM   & $51.14\pm0.18$ & $52.15\pm0.20$ & $63.23\pm0.25$ & $80.11\pm2.68$ & $79.31\pm0.29$ \\
TabSyn    & $50.94\pm0.17$ & $51.20\pm0.28$ & $52.90\pm0.22$ & $50.37\pm0.13$ & $50.85\pm0.33$ \\
TabDiff   & $\mathbf{50.10\pm0.32}$ & $51.11\pm0.36$ & $50.24\pm0.62$ & $\mathbf{50.50\pm0.36}$ & $51.04\pm0.32$ \\
\midrule
\textbf{TabbyFlow} & $50.32\pm0.16$ & $\mathbf{50.82\pm0.27}$ & $\mathbf{50.17\pm0.32}$ & $50.94\pm0.13$ & $\mathbf{50.83\pm0.29}$ \\
\bottomrule
\end{tabular}}
\label{tab:dcr}
\end{table*}

\begin{table*}[t]
\centering
\caption{Machine learning efficiency across datasets.}
\resizebox{\textwidth}{!}{
\begin{tabular}{lcccccccc}
\toprule
\textbf{Methods} & \textbf{Adult (AUC $\uparrow$)} & \textbf{Beijing (RMSE $\downarrow$)} & \textbf{Default (AUC $\uparrow$)} & \textbf{Shoppers (AUC $\uparrow$)} & \textbf{Magic (AUC $\uparrow$)} & \textbf{News (RMSE $\downarrow$)} & \textbf{Average}\\
\midrule
\textit{Real} & $.927 \pm .000$ & $.770 \pm .005$ & $.926 \pm .001$ & $.946 \pm .001$ & $.423 \pm .003$ & $.842 \pm .002$ & $.806$\\
\midrule
CTGAN & $.886 \pm .002$ & $.696 \pm .005$ & $.875 \pm .009$ & $.855 \pm .006$ & $\mathbf{.902 \pm .019}$ & $\mathbf{.880 \pm .016}$  & $.849$\\
TVAE & $.878 \pm .004$ & $.724 \pm .005$ & $.871 \pm .006$ & $.887 \pm .004$ & $.770 \pm .011$ & $.861 \pm .016$   & $.831$\\
GOGGLE & $.778 \pm .012$ & $.584 \pm .005$ & $.658 \pm .052$ & $.654 \pm .024$ & $.709 \pm .025$ & $.877 \pm .002$  & $.710$\\
TabDDPM & $.907 \pm .001$ & $.758 \pm .004$ & $.918 \pm .005$ & $.935 \pm .004$ & $.592 \pm .011$ & $.486 \pm 3.04$  & $.766$\\
TabSyn & $.909 \pm .001$ & $\mathbf{.763 \pm .002}$ & $.914 \pm .004$ & $\mathbf{.937 \pm .002}$ & $.763 \pm .002$ & $.862 \pm .004$  & $\mathbf{.858}$\\
TabDiff & $\mathbf{.912\pm.002}$ & $.763\pm.005$ &  $\mathbf{.921\pm.004}$ & $.936\pm.003$ & $.555\pm.013$ & $.866\pm.021$ & $.826$\\
\midrule
\textbf{TabbyFlow} & $.902 \pm .002$ & $.761 \pm .003$ & $.910 \pm .006$ & $.932 \pm .003$ & $.746 \pm .008$ & $.870 \pm .005$ & $.854$\\
\bottomrule
\end{tabular}}
\label{tab:MLE}
\end{table*}

\paragraph{Downstream Tasks: MLE and Privacy.}
We evaluate the practical utility of synthetic data through MLE, following established protocols \citep{kotelnikov23a, tabsyn, TabDiff}. This metric assesses how well models trained on synthetic data perform on real data, using XGBoost \citep{xgboost} for both classification (measured by AUC) and regression tasks (measured by RMSE). The results in Table \ref{tab:MLE} show that while TabbyFlow performs slightly below TabSyn on MLE scores, it maintains competitive performance in all data sets. This presents an interesting contrast to earlier metrics, where TabbyFlow often outperformed TabSyn. The small performance gap suggests that MLE alone may not fully capture synthetic data quality, particularly given that TabbyFlow achieves comparable results with a substantially simpler architecture than TabSyn's VAE-diffusion framework. On the other hand, we evaluate the performance of the models in preserving privacy as measured by DCR. The results, see in \Cref{tab:dcr}, show that Tabbyflow outperforms all other models on average, albeit marginally.

\paragraph{Summary.}
TabbyFlow demonstrates consistently strong performance across our evaluation framework, spanning both low-order statistics (Shape and Trend) and high-order metrics ($\alpha$-precision and $\beta$-recall), achieving leading or competitive performance compared to state-of-the-art baselines on most datasets. 
The key implication of these results is that TabbyFlow's theoretically motivated exponential family approach can match or exceed the performance of more complex architectures like TabSyn's VAE-diffusion framework, suggesting that careful statistical modeling can be as effective as more elaborate deep learning approaches while offering benefits in terms of simplicity and interoperability.
This is especially attractive when we consider that the current state-of-the-art methods, TabSyn and TabDiff, employ a correction inspired by \citet{10.5555/3600270.3602196}, hence doing twice as many function evaluations as TabbyFlow for the same time discretization. For more details on the inference and hyperparameter training, see \Cref{app:implementation}.

\section{Related Work}\label{sec:related}

Flow matching has been carried out from the continuous to the discrete case through continuous-time Markov chains \citep{gat2024discreteflowmatching}, and through mixtures of Dirichlet distributions \citep{stark2024dirichlet}. These methods differ from the current approach due to the learned sequential sampling used to generate new data and the constraints it imposes on $x$, specifically requiring $ x$ to lie on the simplex. These works have been applied to language generation and DNA sequence design, with no application to data of multiple modalities. More recent work uses gradient-boosted trees to learn the conditional vector field on tabular data \cite{gbt}, showing promising results but a lack of integration of the discrete case. 

In recent years, many generative models for tabular data have been proposed. Some use variational autoencoders (VAE) \cite{Xu2019ModelingTD, liu2023goggle}, while others use generative adversarial networks (GAN) \cite{Xu2019ModelingTD}, diffusion \cite{kotelnikov23a, codi, austin2021structured, TabDiff}, and mixtures of two of the previous frameworks \cite{tabsyn}. Notably, some of these models construct separate diffusion processes for each type of feature, while others project the discrete data by first encoding it in a latent space with a VAE. These new frameworks significantly improve previous alternatives, like \citet{Chawla_2002}, in the quality of fidelity of the generated data. These new approaches further emphasize the importance of modeling the joint probability density function instead of naive approaches. They further show promise in the anonymized use of the generated data in sensitive subjects, such as healthcare. 

However, these models exhibit a common weakness when compared to flow-based models. They have higher complexity in training (e.g., sensitivity to hyperparameters) and require more resources \citep{lipman2023flow}. 
Some initial work has been done on tabular settings using a flow-based model, with a recent example being TabUnite \citep{si2024tabunite}, which proposes unifying the data space and jointly applying a single generative process across all encodings. Another example applies gradient-boosted trees in a diffusion and flow-based setting to generate synthetic data \citep{gbtflow}.

\section{Conclusion}

In this work, we proposed Exponential Family Variational Flow Matching (EF-VFM), a framework that integrates exponential family distributions into the Variational Flow Matching paradigm. 
By leveraging sufficient statistics matching, EF-VFM provides a scalable and probabilistic approach to generative modeling, enabling the efficient generation of mixed continuous and discrete data. 
We established a connection between the EF-VFM objective and Bregman divergences, offering a deeper theoretical understanding of flow matching. 
Our method demonstrated state-of-the-art performance on benchmark tabular datasets, showcasing its ability to handle diverse data distributions while maintaining computational efficiency.

Looking ahead, there are several promising directions for future work. 
One avenue involves exploring the role of information geometry in EF-VFM, particularly leveraging natural gradient descent to better navigate the parameter space of exponential family distributions, which has an implied Fisher-Rao metric. 
Another exciting direction is the extension of EF-VFM to 1) explore more distributions of the exponential family, and 2) to fully match all sufficient statistics, which would enhance the expressiveness of the model. 
Understanding how to incorporate and utilize these sufficient statistics during the generation process is another open challenge with significant potential. 
By addressing these challenges, EF-VFM can further bridge the gap between statistical theory and modern generative modeling, paving the way for more robust and flexible applications.

\paragraph{Acknowledgments} This project was support by the Bosch Center for Artificial Intelligence. JWvdM additionally acknowledges support from the European Union Horizon Framework Programme (Grant agreement ID: 101120237). AGC acknowledges that the presentation of this paper at the conference was financially supported by the Amsterdam ELLIS Unit.

\section*{Impact Statement}
This paper presents work whose goal is to advance the field of Machine Learning. There are many potential societal consequences of our work, none of which we feel must be specifically highlighted here.

\bibliography{bibliography}
\bibliographystyle{icml2025}

\newpage
\appendix
\onecolumn

\section{Proofs}

\subsection{Gradient Exponential Family VFM}
\label{appendix:gradientefvfm}

\gradientefvfm*
\begin{proof}
We know that the Variational Flow Matching objective is defined as follows: 
\begin{equation}
    \mathcal{L}(\theta) 
    = 
    -\mathbb{E}_{t, x_1, x} 
    \left[ 
        \log q_t^\theta(x_1 \mid x)
    \right] 
    = 
    -\mathbb{E}_{t, x_1, x} 
    \left[ 
        \tau(x_1) \cdot \eta_t^{\theta}(x) 
        - 
        A(\eta_t^{\theta}(x))
    \right]
    -
    \mathbb{E}_{x_1} 
    \left[ 
    \log h(x_1)
    \right]    .
\end{equation}
The first term in the expectation is linear, and hence (by the chain rule) we obtain
\begin{equation}
\nabla_{\theta}  
\left(
    \mathbb{E}_{t, x_1, x} 
    \left[ 
        \tau(x_1) \cdot \eta^{\theta}_t(x) 
    \right] 
\right) 
= 
\mathbb{E}_{t, x_1, x} 
\left[ 
    \tau(x_1) \cdot \nabla_{\theta}  \left( \eta^{\theta}_t(x) \right)
\right].
\end{equation}
For the second term, we leverage the fact that in exponential family distributions, the gradient of the log-partition function with respect to the natural parameters equals the expected value of the sufficient statistic, 
\begin{equation}
    \nabla_{\eta} A(\eta)
    \big\vert_{\eta=\eta^\theta(x)} 
    = 
    \mathbb{E}_{q^\theta(x_1 \mid x)}[\tau(x_1)]
    =: 
    \mu_t^\theta(x).
\end{equation}
As such, taking the gradient of the second term results in
\begin{equation}
\mathbb{E}_{t, x_1, x}
\left[
\nabla_{\theta} 
\left( 
    A(\eta^{\theta}_t(x))  
\right)
\right]
=  
\mathbb{E}_{t, x}
\left[
 \mu_t^\theta(x)  \cdot \nabla_{\theta} 
 \left( \eta^{\theta}_t(x) \right)
\right].
\end{equation}
Combining these terms and factoring the gradient of the neural network, we obtain
\begin{equation}
 \nabla_{\theta} \mathcal{L}(\theta) 
 = 
 -
 \mathbb{E}_{t, x_1, x} \left[
    \left(\tau(x_1) - \mu_t^\theta(x) \right)  
    \cdot \nabla_{\theta} \eta_t^{\theta}(x)
\right],
\end{equation}
If we additionally define the moments relative to the posterior probability path as $\mu_t(x):= \mathbb{E}_{p_t(x_1 \mid x)}[\tau(x_1)]$, then we can express the gradient as:
\begin{equation}
 \nabla_{\theta} \mathcal{L}(\theta) 
 = 
 -
 \mathbb{E}_{t, x} \left[
    \left(\mu_t(x) - \mu_t^\theta(x) \right)  
    \cdot \nabla_{\theta} \eta_t^{\theta}(x)
\right],
\end{equation}
completing the proof.\end{proof}

\subsection{Exponential Family VFM as Bregman Divergence}
\label{appendix:efvfmasbregman}
\expfambreg*
\begin{proof}
   Let $q_t^{\theta}(x_1 \mid x) = h(x_1) \exp \left(\tau(x_1) \cdot \eta_t^{\theta}(x)  - A(\eta_t^{\theta}(x)) \right)$ be a member from the exponential family that is regular and minimal. The VFM objective is given by
   \begin{equation}
       \mathcal{L}_{\text{VFM}}(\theta) 
       = 
       -\mathbb{E}_{t, x_1, x} 
       \left[ 
          \tau(x_1) \cdot \eta_t^{\theta}(x)- A(\eta^\theta(x))
       \right] 
       - 
       \mathbb{E}_{x_1}[\log h(x_1)],
   \end{equation}
   By regularity and minimality, we know that the log normalizer $A(\eta)$ is of Legendre type, and as such we can consider its conjugate dual $A^*(\mu)$, where $A^*$ is obtained through the Legendre transform, i.e.
   \begin{equation}
       A^*(\mu) := \sup_{\eta} \{ \mu \cdot \eta - A(\eta) \}.
   \end{equation}
   Conversely, we can define $A(\eta)$ as the Legendre transform of $A(\mu)$,
   \begin{equation}
       A(\eta) := \sup_{\mu} \{ \mu \cdot \eta - A^*(\mu) \}.
   \end{equation}
   From this it follows that $A(\eta)$ and $A^*(\mu)$ define a bijective relationship between $\eta$ and $\mu$,
   \begin{equation}
       \eta(\mu) = \nabla A^*(\mu) \text{ and } \mu(\eta) = \nabla A(\eta),
   \end{equation}
   If we substitute $\eta(\mu)$ into the expression for $A^*(\mu)$, then we obtain,
   \begin{equation}
       A^*(\mu)  = \mu \cdot \eta(\mu) - A(\eta(\mu)).
   \end{equation}
    We now observe that the Bregman divergence $D_{A^*}(\tau(x_1), \mu)$ can be expressed as   
   \begin{align}
        D_{A^*}(\tau(x_1), \mu)
        &= 
        A^*(\tau(x_1))
        - A^*(\mu) 
        - (\tau(x_1) - \mu) \cdot \nabla_\mu A^*(\mu)
        \\
        &=
        A^*(\tau(x_1))
        - A^*(\mu) 
        - (\tau(x_1) - \mu) \cdot \eta(\mu).
   \end{align}
   This means that we can express the VFM objective as
   \begin{align}
    \mathcal{L}_{\text{VFM}}(\theta) 
    &=  
    -
    \mathbb{E}_{t, x_1, x} 
    \left[ 
        \eta_t^\theta(x) \cdot \tau(x_1) - A(\eta^\theta_t(x))
    \right] 
    + 
    \text{const.}
    \\
    &= 
    - 
    \mathbb{E}_{t, x_1, x} 
    \left[ 
        \mu^\theta_t(x) \cdot \eta^\theta_t(x) - A(\eta^\theta_t(x)) 
        + (\tau(x_1) - \mu^\theta_t(x)) \cdot \eta^\theta_t(x) 
    \right] 
    + 
    \text{const.} \\
    &=
    - 
    \mathbb{E}_{t, x_1, x} 
    \left[  
        A^*(\mu_t^\theta(x))  
        + (\tau(x_1) - \mu_t^\theta(x)) 
        \cdot \eta_t^\theta(x)
    \right] + \text{const.} \\
    &= 
    - 
    \mathbb{E}_{t, x_1, x} 
    \left[ 
        - D_{A^*} \left( \tau(x_1), \mu_t^\theta(x) \right) 
        + A^*(\tau(x_1))\right] + \text{const.} \\
    &= 
    \mathbb{E}_{t, x_1, x} \left[ 
        D_{A^*} \left( \tau(x_1), \mu_t^\theta(x) \right) 
    \right] + \text{const.}
    \\
    &= 
    \mathbb{E}_{t, x} \left[ 
        D_{A^*} \left( \mathbb{E}_{p_t(x_1 \mid x)}[\tau(x_1)], \mu_t^\theta(x) \right) 
    \right] + \text{const.}
    \\
    &= 
    \mathbb{E}_{t, x} \left[ 
        D_{A^*} \left( \mu_t(x), \mu_t^\theta(x) \right) 
    \right] + \text{const.}
   \end{align}
   As such, indeed the VFM objective is equal up to a constant independent of $\theta$ to optimizing a Bregman divergence induced by the conjugate dual of the log normalizer, which we wanted to show. 
\end{proof}

\newpage
\section{Implementation Details}
\label{app:implementation}

We perform our experiment on an Nvidia RTX A6000 GPU with 16GB of memory and implement \textsc{TabbyFlow} with PyTorch.

\paragraph{Architecture} 
Following the implementation of \citet{TabDiff,tabsyn}, our setup individually converts each data column into an embedding of dimension 4 using a linear transformation, treating every column equally. 
Next, these embeddings pass through a two-layer transformer along with positional markers. 
After processing, the resulting vectors are merged and then fed into a five-layer feedforward neural network (MLP), which depends on a special timing-based input. 
The final output comes from applying another transformer and a final projection to restore the initial feature dimensions. 
Despite differences in structure, the size and complexity of our model roughly match those of models from prior studies \citep{kotelnikov23a,tabsyn,TabDiff}, largely due to the MLP component.

\paragraph{Hyperparameters}
We keep the same training configuration across all datasets. 
All models train for 8,000 iterations using the Adam optimizer, with batch sizes of 4,096 during training and 10,000 during sampling. 
Similar to \citet{TabDiff}, using a weighting scheme that keeps the categorical loss constant, while gradually reducing the numerical loss weight from one down to zero throughout training, works best.
However, we note that this scheme only reduces the number of epochs necessary to achieve the best results.

\paragraph{Inference}
At the inference stage, we pick the model checkpoint with the lowest training loss. 
Remarkably, the model can achieve SOTA results as reported, requiring as few as 25 steps during inference (T = 25), which is lower than previous SOTA alternatives while using a single function evaluation at each time step.

\paragraph{Data preprocessing} We handle missing values following standard practices \citep{kotelnikov23a, TabDiff, tabsyn}: replacing numerical missing entries with column means and treating categorical missing values as new categories. To stabilize training across diverse numerical scales, we apply \texttt{QuantileTransformer}\textsuperscript{\ref{fn:quantile}} during training and its inverse during sampling.

\paragraph{Data splits} Following \citet{kotelnikov23a,tabsyn, TabDiff}, we split each dataset into "real" and "test" sets. For unconditional generation tasks, models are trained and evaluated on the "real" set. For machine learning efficiency evaluation, we further split the "real" set into training and validation sets, while using the "test" set for final evaluation. \footnote {\label{fn:quantile}\url{https://scikit-learn.org/stable/modules/generated/sklearn.preprocessing.QuantileTransformer.html}}

\section{Data Details}
\label{app:data}

We use six tabular datasets from UCI Machine Learning Repository\textsuperscript{\ref{fn:uci}}: Adult, Default, Shoppers, Magic, Beijing, and News, where each tabular dataset is associated with a machine-learning task. Classification: Adult, Default, Magic, and Shoppers. Regression: Beijing and News. The statistics of the datasets are presented in Table~\ref{tab:dataset_statistics}. \footnote{\label{fn:uci}\url{https://archive.ics.uci.edu/datasets}}

\begin{table}[h!]
    \centering
    \caption{Statistics of datasets. \# Num stands for the number of numerical columns, and \# Cat stands for the number of categorical columns. \# Max Cat stands for the number of categories of the categorical column with the most categories.}
    \label{tab:dataset_statistics}
    \begin{tabular}{lrrrrrrrl}
        \toprule
        \textbf{Dataset} & \textbf{\# Rows} & \textbf{\# Num} & \textbf{\# Cat} & \textbf{\# Max Cat} & \textbf{\# Train} & \textbf{\# Validation} & \textbf{\# Test} & \textbf{Task} \\
        \midrule
        \textbf{Adult}    & 48,842  & 6  & 9  & 42  & 28,943  & 3,618  & 16,281  & Classification \\
        \textbf{Default}  & 30,000  & 14 & 11 & 11  & 24,000  & 3,000  & 3,000   & Classification \\
        \textbf{Shoppers} & 12,330  & 10 & 8  & 20  & 9,864   & 1,233  & 1,233   & Classification \\
        \textbf{Magic}    & 19,019  & 10 & 1  & 2   & 15,215  & 1,902  & 1,902   & Classification \\
        \textbf{Beijing}  & 43,824  & 7  & 5  & 31  & 35,058  & 4,383  & 4,383   & Regression \\
        \textbf{News}     & 39,644  & 46 & 2  & 7   & 31,714  & 3,965  & 3,965   & Regression \\
        \bottomrule
    \end{tabular}
\end{table}

\end{document}